\documentclass[12pt]{article}

\usepackage{amsmath,amssymb,amsthm}
\usepackage{times}
\usepackage[authoryear,round]{natbib}

\usepackage{multirow}
\usepackage{hyperref}
\usepackage{graphicx}
\usepackage{color,colortbl}
\usepackage{bbm}
\usepackage{latexsym}
\usepackage{dsfont}
\usepackage{enumerate}
\usepackage{enumitem}
\usepackage{verbatim}
\usepackage{booktabs,array}
\usepackage{tabu}

\newtheorem{theorem}{Theorem}
\newtheorem{lemma}[theorem]{Lemma}
\newtheorem{corollary}[theorem]{Corollary}
\newtheorem{proposition}[theorem]{Proposition}

\theoremstyle{definition}

\newtheorem{remark}[theorem]{Remark}

\newcommand{\dotcup}{\ensuremath{\mathaccent\cdot\cup}}

\newcommand{\Ecal}{\mathcal{E}}

\newcommand{\Lcal}{\mathcal{L}}
\newcommand{\Mcal}{\mathcal{M}}
\newcommand{\Ncal}{\mathcal{N}}
\newcommand{\Pcal}{\mathcal{P}}
\newcommand{\Scal}{\mathcal{S}}
\newcommand{\Vcal}{\mathcal{V}}
\newcommand{\Xcal}{\mathcal{X}}
\newcommand{\Ycal}{\mathcal{Y}}
\newcommand{\Zcal}{\mathcal{Z}}
\newcommand{\N}{\mathbb{N}}
\newcommand{\R}{\mathbb{R}}

\newcommand{\supp}{\operatorname{supp}}
\newcommand{\RBM}{\ensuremath{\operatorname{RBM}}}
\newcommand{\DBN}{\ensuremath{\operatorname{DBN}}}
\newcommand{\Dir}{\operatorname{Dir}}

\newcommand{\aXu}{\mathbf{x}}
\newcommand{\aYu}{\mathbf{y}}
\newcommand{\aVu}{\mathbf{v}}

\newcommand{\rd}{\ensuremath{\mathrm{d}}}

\newcommand{\captionfonts}{\normalsize}

\makeatletter  
\long\def\@makecaption#1#2{%
  \vskip\abovecaptionskip
  \sbox\@tempboxa{{\captionfonts #1: #2}}%
  \ifdim \wd\@tempboxa >\hsize
    {\captionfonts #1: #2\par}
  \else
    \hbox to\hsize{\hfil\box\@tempboxa\hfil}%
  \fi
  \vskip\belowcaptionskip}
\makeatother   

\makeatletter
\renewenvironment{table}
    {%
     \@float{table}}
    {\end@float}
\makeatother

\begin{document}
\hspace{13.9cm}1

\ \vspace{20mm}\\

{\LARGE Universal Approximation Depth and Errors of Narrow Belief Networks with Discrete Units}

\ \\
{\bf \large Guido F. Mont\'ufar$^{\displaystyle 1}$}\\
{$^{\displaystyle 1}$Department of Mathematics, Pennsylvania State University, University Park, PA 16802, USA.}\\
%

{\bf Keywords:} Deep belief network, restricted Boltzmann machine, universal approximation, representational power, Kullback-Leibler divergence, $q$-ary variable

\thispagestyle{empty}
\markboth{}{NC instructions}
\ \vspace{-0mm}\\
%
\begin{center} {\bf Abstract} \end{center}
We generalize recent theoretical work on the minimal number of layers of narrow deep belief networks that can approximate any probability distribution on the states of their visible units arbitrarily well. 
We relax the setting of binary units~\citep[][]{Hinton2008,LeRoux2008,LeRoux2010,Montufar2011} to units with arbitrary finite state spaces, and the vanishing approximation error to an arbitrary approximation error tolerance.  
For example, we show that a $q$-ary deep belief network with $L\geq 2+\frac{q^{\lceil m- \delta \rceil}-1}{q-1}$ layers of width $n \leq m + \log_q(m) + 1$ for some $m\in\N$ can approximate any probability distribution on $\{0,1,\ldots,q-1\}^n$ without exceeding a Kullback-Leibler divergence of $\delta$. 
Our analysis covers discrete restricted Boltzmann machines and na\"ive Bayes models as special cases. 

\section{Introduction}
A {\em deep belief network} (DBN)~\citep[][]{Hinton2006} is a layered stochastic network with undirected bipartite interactions between the units in the top two layers, and directed bipartite interactions between the units in all other subsequent pairs of layers, directed towards the bottom layer. 
The top two layers form a {\em restricted Boltzmann machine} (RBM)~\citep[][]{Smolensky1986}. 
The entire network defines a model of probability distributions on the states of the units in the bottom layer, the {\em visible} layer. 
When the number of units in every layer has the same order of magnitude, the network is called {\em narrow}. 
The {\em depth} refers to the number of layers. 
Deep network architectures are believed to play a key role in information processing of intelligent agents, see~\citep{Bengio-2009} for an overview on this exciting topic. 
DBNs were the first deep architectures to be envisaged together with an efficient unsupervised training algorithm~\citep{Hinton2006}. 
Due to their restricted connectivity, it is possible to greedily train their layers one at the time, and in this way, identify remarkably good parameter initializations for solving specific tasks~\cite[see][]{Bengio2007}. 
The ability to train deep architectures efficiently has pioneered a great number of applications in machine learning and in the booming field {\em deep learning}. 

The representational power of neural networks has been studied for several decades, whereby their universal approximation properties have received special attention. 
For instance, a well known result~\citep[][]{DBLP:journals/nn/HornikSW89} shows that multilayer feedforward networks with one exponentially large layer of hidden units are universal approximators of Borel measurable functions. 
Although universal approximation has a limited importance for practical purposes,\footnote{Where a more or less good approximation of a small set of target distributions is often sufficient, or where the goal is not to model data directly but rather to obtain abstract representations of data.} it plays an important role as warrant for consistency and sufficiency of the complexity attainable by specific classes of learning systems. 
Besides the universal approximation question, it is natural to ask ``how well is a given network able to approximate certain classes of probability distributions?'' 
This note pursues an account on the ability of DBNs to approximate probability distributions. 

The first universal approximation result for deep and narrow sigmoid belief networks is due to~\citet[][]{Hinton2008}. 
They showed that a narrow sigmoid belief network with $3(2^n-1)+1$ layers can represent probability distributions arbitrarily close to any probability distribution on the set of length $n$ binary vectors. 
Their result shows that not only exponentially wide and shallow networks are universal approximators~\citep{DBLP:journals/nn/HornikSW89}, but also exponentially deep and narrow ones are. 
Subsequent work has studied the optimality question ``how deep is deep enough?,'' with improved universal approximation depth bounds by~\citet[][]{LeRoux2010} and later by~\citet[][]{Montufar2011}, which we will discuss below in more detail. 
These papers focus on the minimal depth of narrow DBN universal approximators with binary units; 
that is, the number of layers that these networks must have in order to be able to represent probability distributions arbitrarily close to any probability distribution on the states of their visible units. 
The present note complements that analysis in two ways: 

First, instead of asking for the minimal size of universal approximators, we ask for the minimal size of networks that can approximate any distribution to a given error tolerance, 
treating the universal approximation problem as the special case of zero error tolerance. 
This analysis gives a theoretical basis on which to balance model accuracy and parameter count. 
For comparison, universal approximation is a binary property which always requires an exponential number of parameters. 
As it turns out, our analysis also allows us to estimate the expected value of the model approximation errors incurred when learning classes of distributions, say low-entropy distributions, with networks of given sizes. 

Second, we consider networks with finite valued units, called discrete or multinomial DBNs, including binary DBNs as special cases. 
Non-binary units serve, obviously, to encode non-binary features directly, which may be interesting in multi-channel perception, e.g., color-temperature-distance sensory inputs. 
Additionally, the interactions between discrete units can carry much richer relations that those between binary units. In particular, within the non-binary discrete setting, DBNs, RBMs, and na\"ive Bayes models can be seen as representatives of the same class of probability models. 

\medskip

This paper is organized as follows. 
Section~\ref{section:preliminaries} gives formal definitions, 
before we proceed to state our main result Theorem~\ref{theorem:main} in Section~\ref{section:main}: a bound on the approximation errors of discrete DBNs. 
A universal approximation depth bound follows directly. 
After this, a discussion of the result is given, together with a sketch of the proof. 
The proof entails several steps of independent interest, developed in the next sections. 
Section~\ref{section:topRBM} addresses the representational power and approximation errors of RBMs with discrete units. 
Section~\ref{section:star} studies the models of conditional distributions represented by feedforward discrete stochastic networks (DBN layers). 
Section~\ref{section:sharing} studies concatenations of layers of feedforward networks and elaborates on the patterns of probability sharing steps (transformations of probability distributions) that they can realize. 
Section~\ref{section:theDBN} concludes the proof of the main theorem and gives a corollary about the expectation value of the approximation error of DBNs. 
Section~\ref{section:experiments} presents an empirical validation scheme and tests the approximation error bounds numerically on small networks. 

\section{Preliminaries}
\label{section:preliminaries}

\begin{figure}
\centering
\includegraphics{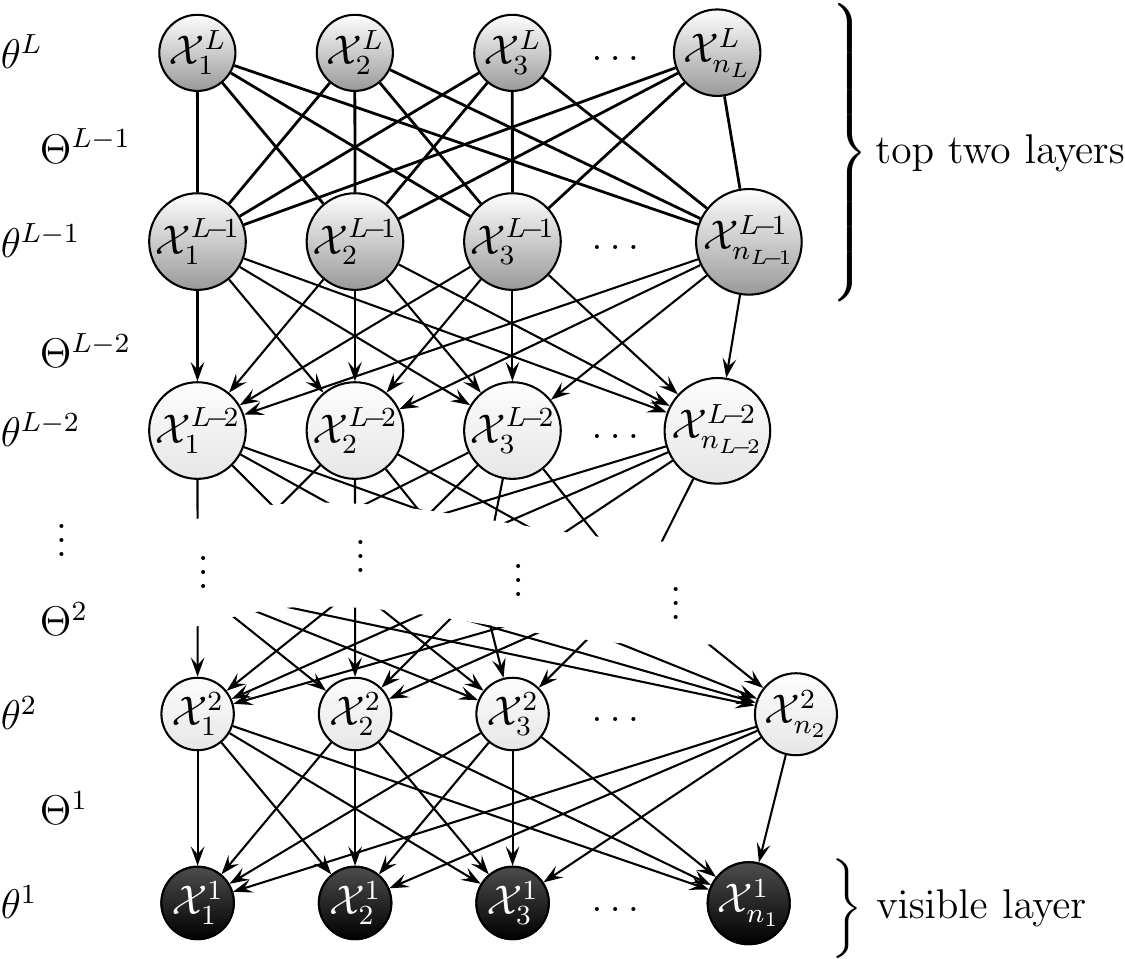}
\caption{Graphical representation of a discrete DBN probability model. 
Each node represents a unit with the indicated state space. 
The top two layers have undirected connections; they correspond to the term $p_{L-1,L}$ described in eq.~\eqref{eq:rbmeqjoint}. 
All other layers receive directed connections, corresponding to the terms $p_l$, $l\in[L-2]$ described in eq.~\eqref{eq:layercondrs}. 
Only the bottom layer is visible. 
}
\label{Figure1}
\end{figure}

A few formal definitions are necessary before proceeding. 
Given a finite set $\Xcal$, we denote $\Delta(\Xcal)$ the set of all probability distributions on $\Xcal$. 
A {\em model} of probability distributions on $\Xcal$ is a subset $\Mcal\subseteq\Delta(\Xcal)$. 
Given a pair of distributions $p,q\in\Delta(\Xcal)$, 
the Kullback-Leibler divergence from $p$ to $q$ is defined as 
$D(p\|q) := \sum_{x\in\Xcal}p(x)\log\frac{p(x)}{q(x)}$ 
when $\supp(p)\!\subseteq\!\supp(q)$, and $D(p\|q)\!:=\!\infty$ otherwise. 
The divergence from a distribution $p$ to a model $\Mcal\subseteq\Delta(\Xcal)$ is defined as $D(p\|\Mcal) :=\inf_{q\in\Mcal} D(p\|q)$.   
The divergence of any distribution on $\Xcal$ to $\Mcal$ is bounded by 
$$D_\Mcal:=\sup_{p\in\Delta(\Xcal)} D(p\|\Mcal).$$ 
We refer to $D_\Mcal$ as the {\em  universal} or  {\em maximal} {\em approximation error} of $\Mcal$. 
The model $\Mcal$ is called a {\em universal approximator} 
of probability distributions on $\Xcal$ iff $D_\Mcal=0$. 

\medskip 

A discrete DBN probability model is specified by a number of layers (the depth of the network), 
the number of units in each layer (the width of each layer), 
and the state space of each unit in each layer. 
Let $L\in\N$, $L\geq2$ be the number of layers. 
We imagine these layers arranged as a stack with layer $1$ at the bottom (this will be the visible layer) and layer $L$ at the top (this will be the {\em deepest} layer). See Figure~\ref{Figure1}. 
For each $l\in\{1,\ldots, L\}=:[L]$, 
let $n_l\in\N$ be the number of units in layer $l$. 
For each $i\in [n_l]$, let $\Xcal_{i}^{l}$, $|\Xcal_i^l|<\infty$ be the state space of unit $i$ in layer $l$. 
We denote the joint state space of the units in layer $l$ by  $\Xcal^l=\Xcal^l_1\times\cdots\times\Xcal^l_{n_l}$, and write $x^l=(x^l_1,\ldots,x^l_{n_l})$ for a state from $\Xcal^l$. 
We call a unit {\em $q$-valued} or {\em $q$-ary} if its state space has cardinality $q$, and assume that $q$ is a finite integer larger than one. 

In order to proceed with the definition of the DBN model, 
we consider the mixed graphical model with undirected connections between the units in the top two layers $L$ and $L-1$, 
and directed connections from the units in layer $l+1$ to the units in layer $l$ for all $l\in[L-2]$. 
This model consists of joint probability distributions on the states  $\Xcal=\Xcal^1\times\cdots \Xcal^L$ of all network units, parametrized by 
a collection of real matrices and vectors $\Theta = \{ \Theta^1, \ldots, \Theta^{L-1},\theta^1,\ldots, \theta^L  \}$. 
For each $l\in[L-1]$, the matrix $\Theta^l$ contains the interaction weights between units in layers $l$ and $l+1$. 
It consists of row blocks $\Theta^l_i\in\R^{(|\Xcal^l_i|-1)\times (\sum_{j\in[n_{l+1}]} (|\Xcal^{l+1}_j|-1))}$ for all $i\in[n_l]$. 
For each $l\in[L]$, the row vector $\theta^{l}$ contains the bias weights of the units in layer $l$. 
It consists of blocks $\theta^l_i\in\R^{|\Xcal^l_i|-1}$ for all $i\in[n_l]$. 

Note that the bias of a unit with state space $\Xcal^l_i$ is a vector with $|\Xcal^l_i|-1$ entries, 
and the interaction of a pair of units with state spaces $\Xcal_i^l$ and $\Xcal_j^{l+1}$ is described by a matrix of order $(|\Xcal_i^l|-1)\times (|\Xcal_j^{l+1}|-1)$. 
The number of interaction and bias parameters in the entire network adds to $\sum_{l=1}^{L-1} ( \sum_{i\in[n_l]}(|\Xcal^l_i| - 1) )(1+ \sum_{j\in[n_{l+1}]}(|\Xcal^{l+1}_j| - 1) )  +  \sum_{i\in[n_L]}(|\Xcal^L_i| - 1)$. 

For any choice $\Theta$ of these parameters, 
the corresponding probability distribution on the states of all units is 
\begin{multline}
p(x^1,\ldots,x^L; \Theta) \; = \; 
p_{L-1,L}(x^{L-1},x^L; \Theta^{L-1},\theta^{L-1},\theta^L)
\prod_{l=1}^{L-2}p_l(x^l|x^{l+1};\Theta^l,\theta^l) \\\quad \text{for all $(x^1,\ldots,x^L)\in\Xcal^1\times\cdots\times\Xcal^L$}; \label{eq:jointeq} 
\end{multline}
where 
\begin{multline}
p_{L-1,L}(x ,y ; \Theta^{L-1},\theta^{L-1},\theta^L) \;  = \; 
 \frac{\exp( \aXu^\top \,\Theta^{L-1} \,\aYu  + \theta^{L-1} \,\aXu + \theta^{L} \,\aYu)}{Z(\Theta^{L-1}, \theta^{L-1},\theta^L)} \\ \quad\text{for all $(x,y)\in\Xcal^{L-1}\times\Xcal^L$}; \label{eq:rbmeqjoint} 
\end{multline}
and  
\begin{multline}
p_l(x |y ; \Theta^{l},\theta^l) \; = \;   \prod_{i\in[n_l]} p_{l,i}(x_i|y ;\Theta_i^l, \theta_i^l ) \\
\quad\text{for all $x\in\Xcal^l$ and $y\in\Xcal^{l+1}$; \quad for each $l\in[L-2]$}; \label{eq:layercondrs} 
\end{multline} 
with factors given by 
\begin{gather}
p_{l,i}(x_i|y ;\Theta_i^l, \theta_i^l ) \; = \; \frac{\exp( \aXu_i^\top \, \Theta^l_i \, \aYu  +  \theta^l_i \,\aXu_i)}{Z(\Theta^l_i \,\aYu, \theta^l_i )}   
\quad\text{for all $x_i\in\Xcal^l_i$ and $y\in\Xcal^{l+1}$}. \label{eq:starcond}
\end{gather}
Here we use following notation. 
Given a state vector $x=(x_1,\ldots,x_n)$ of $n$ units with joint state space $\Xcal_1\times\cdots\times\Xcal_n=\{0,1,\ldots, q_1-1\}\times\cdots\times\{0,1,\ldots, q_n-1\}$, 
$\aXu$ denotes the $x$-th column of a minimal matrix of sufficient statistics for the independent distributions of these $n$ units. 
To make this more concrete, 
we set $\aXu$ equal to a column vector with blocks $\aXu_1$, \ldots,  $\aXu_n$, where $\aXu_i=(\delta_{y_i}(x_i))_{y_i\in\Xcal_i\setminus\{0\}}$ is the one-hot representation of $x_i$ without the first entry, for all $i\in[n]$. 
For example, if $ x=(x_1,x_2)=(1,0) \in\Xcal_1\times\Xcal_2=\{0,1,2\}\times\{0,1,2\}$, then 
$\aXu=\left[\begin{smallmatrix} \aXu_1\\\aXu_2\end{smallmatrix}\right]$, 
with $\aXu_1= \left[\begin{smallmatrix}1\\0 \end{smallmatrix}\right]$ and $\aXu_2= \left[\begin{smallmatrix}0\\0\end{smallmatrix}\right]$. 

The function 
\begin{align}
Z(\Theta^{L-1}, \theta^{L-1},\theta^L) &\; = \;
\sum_{x\in\Xcal^{L-1}, y\in\Xcal^L} \exp( \aXu^\top \,\Theta^{L-1} \,\aYu + \theta^{L-1} \,\aXu + \theta^{L} \,\aYu)
\intertext{normalizes the probability distribution $p_{L-1,L}(\cdot;\Theta^{L-1},\theta^{L-1},\theta^{L})\in\Delta(\Xcal^{L-1}\times\Xcal^L)$ from eq.~\eqref{eq:rbmeqjoint}. 
Likewise, the function }
Z(\Theta_i^l \aYu, \theta^l_i) & \; = \; \sum_{x_i\in\Xcal^l_i}\exp( \aXu_i^\top \, \Theta^l_i \, \aYu  + \theta_i^{l} \,\aXu )
\end{align}
normalizes the probability distribution $p_{l,i}(\cdot|y;\Theta_i^l, \theta_i^l)\in\Delta(\Xcal^l_i)$ from eq.~\eqref{eq:starcond} for each $i\in[n_l]$ and $l\in[L-2]$. 

The  marginal of the distribution 
$p(\cdot;\Theta)\in\Delta(\Xcal^1\times\cdots\times\Xcal^L)$ from eq.~\eqref{eq:jointeq} on the states $\Xcal^1$ of the units in the first layer is given by 
\begin{equation}
P(x^1 ; \Theta) \; = \; \sum_{( x^2,\ldots,x^L)\in\Xcal^2\times\cdots\times\Xcal^L}p(x^1,\ldots,x^L ; \Theta)\quad\text{for all $x^1\in\Xcal^1$}. \label{eq:visdistr}
\end{equation}
The discrete DBN probability model with $L$ layers of widths $n_1,\ldots, n_L$ and state spaces $\Xcal^1, \ldots, \Xcal^L$, 
is the set of probability distributions 
$P(\cdot ; \Theta)\in\Delta(\Xcal^1)$ expressible by eq.~\eqref{eq:visdistr} for all possible choices of the parameter $\Theta$. 
Intuitively, this set is a linear projection of a manifold parametrized by $\Theta$, and may have self-intersections or other singularities. 

 \medskip 

The discrete DBN probability model with $L=2$ is a discrete RBM probability model. \label{page:RBMdef} 
This model consists of the marginal distributions on $\Xcal^{L-1}$ of the distributions $p_{L-1,L}(\cdot;\Theta^{L-1},\theta^{L-1},\theta^L)$ from eq.~\eqref{eq:rbmeqjoint} for all possible choices of $\Theta^{L-1}$, $\theta^{L-1}$, and $\theta^L$. 

When $L>2$, the distributions on $\Xcal^{L-1}$ defined by the top two DBN layers can be seen as the inputs of the 
stochastic maps defined by the conditional distributions $p_{L-2}(\cdot |\cdot;\Theta^{L-2},\theta^{L-2})$ from eq.~\eqref{eq:layercondrs}. The outputs of these maps are probability distributions on $\Xcal^{L-2}$
that can be seen as the inputs of the stochastic maps defined by the next lower layer and so forth. 
The discrete DBN probability model can be seen as the set of images of a discrete RBM probability model by a family of sequences of stochastic maps. 

\medskip

The following simple class of probability models will be useful to study the approximation capabilities of DBN models. 
Let $\varrho=\{A_1,\ldots,A_N\}$ be a partition of a finite set $\Xcal$. 
The {\em partition model} $\Pcal$ with partition $\varrho$ is the set of probability distributions on $\Xcal$ which have constant value on each $A_i$. 
Geometrically, this is  the simplex with vertices $\mathds{1}_{A_i}/|A_i|$ for all $i\in[N]$, where $\mathds{1}_{A_i}$ is the indicator function of $A_i$. 
The {\em coarseness} of $\Pcal$ is $\max_i |A_i|$. 
Unlike many statistical models, partition models have a well understood Kullback-Leibler divergence. 
If $\Pcal$ is a partition model of coarseness $c$, then $D_{\Pcal} = \log(c)$. \label{lem:part-mod-max-KL} 
Furthermore, partition models are known to be optimally approximating exponential families, in the sense that they minimize the universal approximation error among all closures of exponential families of a given dimension~\citep[see][]{Rauh13:Optimal_Expfams}. 

\section{Main Result}
\label{section:main}

The starting point of our considerations is the following result for binary DBNs:  

\begin{theorem}
\label{theorem:1}
A deep belief network probability model with $L$ layers of binary units of width $n =2^{k-1}+k$ (for some $k\in\N$) is a universal approximator of probability distributions on $\{0,1\}^n$ whenever 
$L\geq  1 + 2^{2^{k-1}}$. 
\end{theorem}
Note that 
\begin{equation}
\frac{2^n}{2 (n- \log_2(n))} \leq 2^{2^{k-1}} \leq  \frac{2^n}{2( n- \log_2(n)-1 )}.
\end{equation} 
This result is due to~\citet[][Theorem~2]{Montufar2011}. 
It is based on  a refinement of previous work by~\citet[][]{LeRoux2010}, who  obtained the bound $L\geq 1+ \frac{2^n}{n}$ when $n$ is a power of two. 

The main result of this paper is following generalization of Theorem~\ref{theorem:1}. 
Here we make the simplifying assumption that all layers have the same width $n$ 
and the same state space. 
The result holds automatically for DBNs with wider hidden layers or hidden units with larger state spaces. 

\begin{theorem}
\label{theorem:main}
Let $\DBN$ be a deep belief network probability model with $L\in\N$, $L\geq 2$ layers of width $n\in\N$. 
Let the $i$-th unit of each layer have state space $\{0,1,\ldots, q_i-1\}$, $q_i\in\N$, $2\leq q_i<\infty$, for each $i\in[n]$. 
Let  $m$ be any integer with $n\geq m\geq\prod_{j=m+2}^{n}q_j$, and let $q =q_1\geq \cdots\geq q_m$. 
If $L\geq 2+\frac{q^S-1}{q-1}$ for some $S\in\{0,1,\ldots,m\}$, 
then the probability model $\DBN$ can approximate each element of a partition model of coarseness $\prod_{j\in[{m-S}]}q_j$ arbitrarily well. 
The Kullback-Leibler divergence from any distribution on $\{0,1,\ldots,q_1-1\}\times\cdots\times\{0,1,\ldots,q_n-1\}$ to $\DBN$ is bounded by 
\begin{equation*}
D_{\DBN} \leq \log (\prod_{j\in[m-S]} q_j). 
\end{equation*}
\noindent
In particular, this DBN probability model is a universal approximator whenever 
\begin{equation*}
L\geq 2+\frac{q^m-1}{q-1}. 
\end{equation*}
\end{theorem}
When all units are $q$-ary and the layer width is $n= q^{k-1} + k$ for some $k\in\N$, then the DBN probability model is a universal approximator of distributions on $\{0,1,\ldots, q-1 \}^n$ whenever $L\geq 2+\frac{q^{q^{k-1}}-1}{q-1}$. Note that 
\begin{equation}
\frac{q^n -1}{q(q-1)(n-\log_q(n))}\leq \frac{q^{q^{k-1}}-1}{q-1} \leq \frac{ q^n-1}{q(q-1)(n -\log_q(n)-1)}.  
\end{equation} 
The theorem is illustrated in Figure~\ref{Figure2}. 

\begin{figure}
\centering
\includegraphics[scale=1.1]{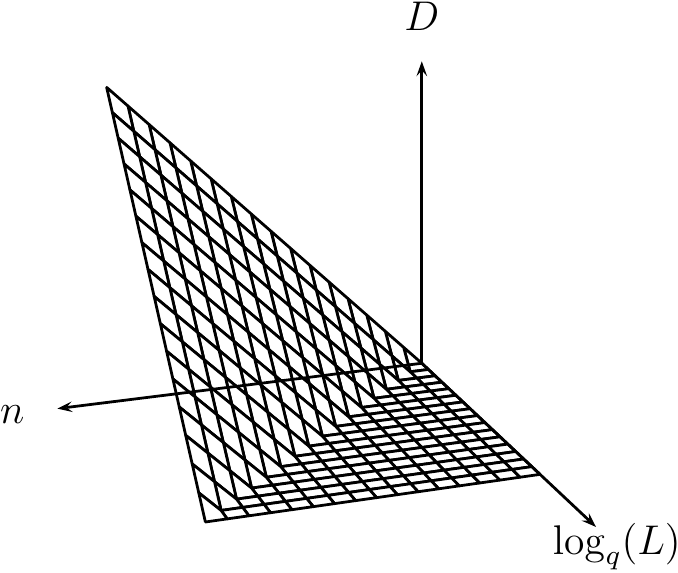} 
\caption{Qualitative illustration of Theorem~\ref{theorem:main}. 
Shown is the large-scale behaviour of the DBN universal approximation error  upper-bound as a function $D$ of  the layer width  $n$ and the logarithm of the number of layers $\log_q(L)$. 
Here it is assumed that the Kullback-Leibler divergence is computed in base $q$ logarithm and that all units are $q$-ary. 
The number of parameters of these DBNs scales with $Ln^2(q-1)^2$. 
}
\label{Figure2}
\end{figure}

\subsection*{Remarks}

The number of parameters of a $q$-ary DBN with $L$ layers of width $n$ is $(L-1)(n(q-1)+1)n(q-1)+n(q-1)$. 
Since the set of probability distributions on $\{0,1,\ldots,q-1\}^n$ has dimension $q^n -1$, the DBN model is full dimensional only if $L \geq\frac{q^n-1}{ n(q-1)(n(q-1)+2)}+1$. 
This is a parameter-counting lower bound for the universal approximation depth. 
Theorem~\ref{theorem:main} gives an upper bound for the minimal universal approximation depth. 
The upper bound from the theorem surpasses the parameter-counting lower bound by roughly a factor $n$. 
We think that the upper bound is tight, up to sub-linear factors, in consideration of the following. 
Probability models with hidden variables can have dimension strictly smaller than their parameter count (dimension defect). 
Moreover, in some cases even  full dimensional models represent only very restricted classes of distributions, as has been observed, for example, in binary tree models with hidden variables. 
It is known that for any prime power $q$, the smallest na\"ive Bayes model universal approximator of distributions on $\{0,1,\ldots,q-1\}^n$ has $q^{n-1}(n(q-1) +1) -1$ parameters~\citep[see][Theorem~13]{Montufar2010a}. 
Hence for these models the number of parameters needed to achieve universal approximation surpasses the corresponding parameter-counting lower bound $q^n/(n(q-1)+1)$ by a factor of order $n$. 

Computing tight bounds for the maximum of the Kullback-Leibler divergence is a notoriously challenging problem. 
This is even so for simple probability models without hidden variables, like independence models with mixed discrete variables. 
The optimality of our DBN error bounds is not completely settled at this point, but we think that they give a close description of the large-scale approximation error behaviour of DBNs. 
For the limiting case of one single layer with $n$ independent $q$-ary units, it is known that the maximal divergence is equal to $(n-1)\log(q)$~\citep[see][]{AyKnauf06:Maximizing_Multiinformation}, corresponding to the line $\log_q(L)=0$ in Figure~\ref{Figure2}. 
Furthermore, when our upper bounds vanish, they obviously are tight (corresponding to the points with value zero in Figure~\ref{Figure2}). 

Discrete DBNs have many hyperparameters (the layer widths and the state spaces of the units), which makes their analysis combinatorially intricate. 
Some of these intricacies are apparent from the floor and ceiling functions in our main theorem. 
This theorem tries to balance accuracy, generality, and clarity. 
In some cases, the bounds can be improved by exhausting 
the representational power gain per layer described in Theorem~\ref{theorem:sharing}. 
A more detailed and accurate account on the two-layer case (RBMs) is given in Section~\ref{section:topRBM}. 
In Section~\ref{section:theDBN} we give results describing probability distributions contained in the DBN model (Proposition~\ref{proposition:distrDBN}) and addressing the expectation value of the divergence (Corollary~\ref{corollary:expectations}). 
Section~\ref{section:experiments} contains an empirical discussion, together with the numerical evaluation of small models. 

\subsection*{Outline of the Proof}

We will prove Theorem~\ref{theorem:main} by first studying the individual parts of the DBN: the RBM formed by the top two layers (Section~\ref{section:topRBM}); the individual units with directed inputs (Section~\ref{section:star}); the {\em probability sharing} realized by stacks of layers (Section~\ref{section:sharing}); 
and finally, 
the sets of distributions of the units in the bottom layer (Section~\ref{section:theDBN}). 
The proof steps can be summarized as follows: 
\begin{list}{\labelitemi}{\leftmargin=1em}
\item Show that the top RBM can approximate any probability distribution with support on a set of the form $\Xcal_1\times\cdots\Xcal_k\times\underset{k+1}{\{0\}}\times\cdots\times\underset{n}{\{0\}}$ arbitrarily well. 

\item For a unit with state space $\Xcal_1$ receiving $n$ directed inputs, show that there is a choice of parameters for which the following holds for each state $h_n\in\Xcal_n$ of the $n$-th input unit: If the input vector is $(h_1,h_2,\ldots,h_n)$, then the unit outputs $h_1'$ with probability $p^{h_n}(h_1')$, where $p^{h_n}$ is an arbitrary distribution on $\Xcal_1$ for all $h_n\in\Xcal_n$. 

\item Show that there is a sequence of $\frac{q^m-1}{q-1}$ stochastic maps $p(h)\mapsto p(v)=\sum_h p(v|h)p(h)$ each of which superposes nearly $q n$ probability multi-sharing steps, which maps the probability distributions represented by the top RBM to an arbitrary probability distribution on $\Xcal_1\times\cdots\times\Xcal_n$. 

\item Show that the DBN approximates certain classes of tractable probability distributions arbitrarily well, and estimate their maximal approximation errors. 
\end{list}

The superposition of probability sharing steps is inspired by~\citep[][]{LeRoux2010}, together with the refinements of that work devised in~\citep[][]{Montufar2011}. 
By {\em probability sharing} we refer to the process of transferring an arbitrary amount of probability from a state vector $x'$ to another state vector $x''$. 
In contrast to the binary proofs, where each layer superposes about $2n$ sharing steps, here each layer superposes about $q n$ multi-sharing steps, whereby each multi-sharing step transfers probability from one state to $q-1$ states (when the units are $q$-ary). 
With this, a more general treatment of models of conditional distributions is required. 
Further, additional considerations are required in order to derive tractable submodels of probability distributions which allow to bound the DBN model approximation errors. 

\section{Restricted Boltzmann Machines}
\label{section:topRBM}

We denote by $\RBM_{\Xcal,\Ycal}$ the restricted Boltzmann machine probability model with hidden units $Y_1,\ldots,Y_m$ taking states in $\Ycal=\Ycal_1\times\cdots\times\Ycal_m$ and visible units $X_1,\ldots,X_n$ taking states in $\Xcal=\Xcal_1\times\cdots\times\Xcal_n$. 
Recall the definitions made in pg.~\pageref{page:RBMdef}. 
In the literature RBMs are defined by default with binary units; however, RBMs with discrete  units have appeared  in~\citep[][]{welling:exponential}, and their representational power has been studied in~\citep[][]{montufar2013discrete}.
The results from this section are closely related to the analysis given in~\citep[][]{montufar2013discrete}.  

\begin{theorem} 
\label{theorem:RBMs}
The model $\RBM_{\Xcal,\Ycal}$ can approximate any mixture distribution $p=\sum_{i=0}^m \lambda_i p_i$ arbitrarily well, where $p_0$ is any product distribution, and $p_i$ is any mixture of $(|\Ycal_i|-1)$ product distributions for all $i\in[m]$ satisfying $\supp(p_i)\cap\supp(p_j)=\emptyset$ for all $1\leq i < j\leq m$. 
\end{theorem}
Here, a {\em product distribution} $p$ is a probability distribution on $\Xcal=\Xcal_1 \times\cdots\times\Xcal_n$ that factorizes as  $p(x_1,\ldots,x_n)=\prod_{j\in[n]} p_{j}(x_j)$ for all $x\in\Xcal$, where $p_{j}$ is a distribution on $\Xcal_j$ for all $j\in[n]$. 
A {\em mixture} is a weighted sum  with non-negative weights adding to one. The support of a distribution $p$ is $\supp(p):=\{x\in\Xcal\colon p(x)>0\}$. 

\begin{proof}[Proof of Theorem~\ref{theorem:RBMs}]
Let $\Ecal_\Xcal$ denote the set of strictly positive product distributions of $X_1,\ldots,X_n$. 
Let $\Mcal^k_\Xcal$ denote the set of all mixtures of $k$ product distributions from $\Ecal_\Xcal$. 
The closure $\overline{\Mcal^k_\Xcal}$ contains all mixtures of $k$ product distributions, including those which are not strictly positive. Let $q\circ q'$ denote the renormalized entry-wise product with $(q\circ q')(x)=q(x)q'(x)/\sum_{x'\in\Xcal}q(x')q'(x')$ for all $x\in\Xcal$. 
Let $\mathds{1}$ denote the constant function on $\Xcal$ with value $1$. The model $\RBM_{\Xcal,\Ycal}$ can be written, up to normalization, as the set  
\begin{multline}
\Mcal_\Xcal^{|\Ycal_1|}\circ\cdots \circ \Mcal_\Xcal^{|\Ycal_m|}=
\R_+\Ecal_\Xcal\circ(\mathds{1} + \R_+ \Mcal^{|\Ycal_1|-1}_\Xcal)\circ\cdots\circ(\mathds{1} + \R_+\Mcal^{|\Ycal_m|-1}_\Xcal).   
\end{multline}
Now consider any probability distributions $p_0\in {\Ecal_\Xcal}$, $p_1'\in\overline{\Mcal^{|\Ycal_1|-1}_\Xcal}$, \ldots, $p_m'\in\overline{\Mcal^{|\Ycal_m|-1}_\Xcal}$. 
If $\supp(p_i')\cap\supp(p_j')=\emptyset$ for all $1\leq i<j\leq m$, 
then the product $(\mathds{1} + \lambda_1' p_1')\circ\cdots\circ(\mathds{1} + \lambda_m' p_m')$ is equal to $\mathds{1} + \sum_{i\in[m]}\lambda_i' p_i'$, up to normalization. 
Let $\lambda_i'=\lambda_i / \lambda_0\sum_{x} p_i'(x) p_0(x)$ and $p_i'(x) = p_i(x)/ p_0(x)$. 
Then $\lambda_0p_0\circ(\mathds{1}+\sum_{i\in[m]}\lambda_i p_i')=\sum_{i=0}^m\lambda_i p_i = p$. 
Hence the mixture distribution $p$ is contained in the closure of the RBM model. 
\end{proof}

RBMs can approximate certain partition models arbitrarily well: 

\begin{lemma}
\label{lemma:distrRBMs}
Let $\Pcal$ be the partition model with partition blocks  $\{x_1\}\times\cdots\times\{x_k\}\times\Xcal_{k+1}\times\cdots\times\Xcal_n$ for all $(x_1,\ldots,x_k)\in\Xcal_{1}\times\cdots\times\Xcal_k$. If $ 1+\sum_{j\in[m]}(|\Ycal_j|-1) \geq (\prod_{i\in [k]}|\Xcal_i| ) / \max_{j\in[k]}|\Xcal_j|$, then each distribution contained in $\Pcal$ can be approximated arbitrarily well by distributions from $\RBM_{\Xcal,\Ycal}$. 
\end{lemma}
\begin{proof}
Any point in $\Pcal$ is a mixture of the uniform distributions on the partition blocks. These mixture components have disjoint supports, since the partition blocks are disjoint. They are product distributions, since they can be written as $p_{x_1,\ldots,x_k}=\prod_{i\in[k]}\delta_{x_i} \prod_{i\in[n]\setminus[k]}u_i$, where $u_i$ denotes the uniform distribution on $\Xcal_i$. 
For any  $j\in[k]$, any mixture of the form $\sum_{x_j\in\Xcal_j} \lambda_{x_j} p_{x_1,\ldots,x_k}$ is also a product distribution which factorizes as 
\begin{equation}
(\sum_{x_j\in\Xcal_j} \lambda_{x_j}\delta_{x_j}) \prod_{i\in[k]\setminus\{j\}}\delta_{x_i}
 \prod_{i\in[n]\setminus[k]}u_i. \label{eq:mixtprod}
 \end{equation}
Hence  any point in $\Pcal$ is a mixture of $(\prod_{i\in[k]}|\Xcal_i|)/\max_{j\in[k]}|\Xcal_j|$ product distributions of the form given in eq.~\eqref{eq:mixtprod}. 
The claim follows from Theorem~\ref{theorem:RBMs}. 
\end{proof}

Lemma~\ref{lemma:distrRBMs}, together with the divergence formula for partition models given in pg.~\pageref{lem:part-mod-max-KL}, implies: 

\begin{theorem}
\label{theorem:univRBMs} 
If $1+\sum_{j\in[m]}(|\Ycal_j|-1) \geq \big(\prod_{i\in \Lambda}|\Xcal_i|\big) / \max_{i'\in \Lambda}|\Xcal_{i'}|$ for some $\Lambda\subseteq[n]$, then 
\begin{equation*}
D_{\RBM_{\Xcal,\Ycal}}\leq \log\big(\prod_{i\in[n]\setminus \Lambda}|\Xcal_i| \big). 
\end{equation*}
\noindent
In particular, the model $\RBM_{\Xcal,\Ycal}$ is a universal approximator whenever 
\begin{equation*}
1+\sum_{j\in[m]}(|\Ycal_j|-1) \;\geq\; |\Xcal| / \max_{i\in[n]}|\Xcal_i|. 
\end{equation*}
\end{theorem}
When all units are $q$-ary, the RBM with $(q^{n-1} -1)/(q-1)$ hidden units is a universal approximator of distributions on $\{0,1,\ldots,q-1\}^n$. 
Theorem~\ref{theorem:univRBMs} generalizes previous results on binary RBMs~\citep[][Theorem~1]{Montufar2011} and~\citep[][Theorem~5.1]{NIPS2011_0307}, where it is shown that a binary RBM with $2^{n-1}-1$ hidden units is a universal approximator of distributions on $\{0,1\}^n$ and that the maximal approximation error of binary RBMs decreases at least logarithmically in the number of hidden units. 
A previous result by~\citet[Section~2.5]{Freund1992} shows that a binary RBM with $2^n$ hidden units is a universal approximator of distributions on $\{0,1\}^n$. 
See also the work by~\citet[Theorem~2]{LeRoux2008}. 

\section{The Internal Node of a Star}
\label{section:star}

Consider an inwards directed star graph with leaf variables taking states in $\Ycal=\Ycal_1\times\cdots\times\Ycal_m$  and internal node variable taking states in $\Vcal$. 
Denote by $\Scal_{\Vcal,\Ycal}$ the set of conditional distributions on $\Vcal$ given the states $y\in\Ycal$ of the leaf units, defined by this network. 
Each of these distributions can be written as 
\begin{equation}
p(v|y;\Theta) =   \exp( \aVu^\top \,\Theta \left[ \begin{smallmatrix}1\\\aYu\end{smallmatrix}\right] ) / {Z(\Theta \left[ \begin{smallmatrix}1\\\aYu\end{smallmatrix}\right])} , \quad\text{for all $v\in\Vcal$ and $y\in\Ycal$}. \label{eq:starcds} 
\end{equation}
The distributions from eq.~\eqref{eq:starcond} are of this form, with  $\Theta$ corresponding to $[ (\theta^l_i)^\top | \Theta^l_i]$. 

A conditional distribution $p(\cdot|\cdot)$ is naturally identified with the stochastic map defined by the matrix $(p(x|y))_{y,x}$. 
The following lemma describes some stochastic maps that are representable by the model $\Scal_{\Vcal,\Ycal}$, and which we will use to define a probability sharing scheme in Section~\ref{section:sharing}. 

\begin{lemma}
\label{lemma:star}
Let $\Zcal=\{y_1\}\times\cdots\times\{y_{k-1}\}\times\Ycal_k\times\{y_{k+1}\}\times\cdots\times\{y_m\}\subseteq\Ycal$, $k\neq m$. 
Furthermore, let $\Vcal=\Ycal_m$, and let $\{q^z\colon z\in\Zcal\}$ be any distributions on $\Vcal$. Then there is a choice of the parameters $\Theta$ of $\Scal_{\Vcal,\Ycal}$ for which 
\begin{equation*}
p(\cdot|y;\Theta)=\begin{cases} 
q^y , & \text{ if }  y\in\Zcal \\
\delta_{y_m} ,& \text{ otherwise } 
\end{cases}. 
\end{equation*} 
\end{lemma}

\begin{proof} 
Let $\Ycal_j=\{0,1,\ldots,r_j-1\}$ for all $j\in[m]$, and $r=|\Vcal|=r_m$. 
The set  of strictly positive probability distributions on $\Vcal$ is an exponential family $\Ecal_\Vcal = \{p(v;\theta)=\exp(\aVu^\top\,\theta ) / Z(\theta) \text{ for all $v\in\Vcal$ }\colon \theta\in\R^{d} \}$ with $d=r-1$. 
For some $v\in\Vcal$ let $\vartheta_v\in\R^d$ be the parameter vector of a distribution which attains a unique maximum at $v$. 
Then for any fixed $\eta\in\R^d$ we have 
\begin{equation}
\lim_{K\to\infty}p(x; \eta+K \vartheta_v)=\delta_v(x)\quad\text{for all $x\in\Vcal$}. 
\end{equation} 
To see this, note that $p(x ; K \vartheta_v) \propto p(x ; \vartheta_v)^K$ and hence $\lim_{K\to\infty} p(x ; K \vartheta_v) = \delta_v$. 
Furthermore, $p (x ;\eta+K \vartheta_v ) \propto  p (x ;\eta  ) p (x ; K \vartheta_v )$. 

Without loss of generality let $\Zcal=\Ycal_1\times\{0\}\times\cdots\times\{0\}$. 
For each $z=(z_1,\ldots,z_m) \in\Zcal$ let $\theta^{z_1}\in\R^d$ be such that $p(v ; \theta^{z_1}) = q^z(v)$ for all $v\in\Vcal$. 
The matrix $\Theta $ can be set as follows: 

\begin{equation}
\Theta  =
\left[\begin{array}{c|c|c|c| c}
\theta^0 & \Theta_1 & \Theta_2 & \cdots &\Theta_m 
\end{array}
\right]; 
\end{equation}
where  $\Theta_j$ contains the columns  corresponding to $\aYu_j$ in eq.~\eqref{eq:starcds} and 
\begin{equation}
\newcommand{\mli}{\,\middle|\,}
\renewcommand\arraystretch{1.3}
\renewcommand\arraycolsep{2pt}
\begin{array}{l l l}
\Theta_1 &=
{\left[  \theta^1 - \theta^0 \mli \cdots \mli \theta^{r_1-1} - \theta^0 \right]}&
{\in\R^{d \times r_1};}\\
\Theta_j&=
{	\left[ K_0 \vartheta_0 \mli K_0 \vartheta_0 \mli \cdots \mli K_0 \vartheta_0 \right]
}&
{\in\R^{d \times r_j}, \;\;\text{for  $j=2,\ldots,m-1$};}\\ 
\Theta_m&=
{\left[ K_1 \vartheta_1 \mli K_2 \vartheta_2 \mli \cdots \mli K_{r-1} \vartheta_{r-1}\right]}&
{\in\R^{d \times r}.}
\end{array}
\end{equation}
The matrix $\Theta$ maps  $\{\left[ \begin{smallmatrix}1\\\aYu\end{smallmatrix}\right] \colon y\in\Zcal\}$ to the parameter vectors $\{\theta^{z_1}\colon z_1\in\Ycal_1\}$ with corresponding  distributions $\{q^z\colon z\in\Zcal\}$. 
When  $K_0,\ldots,K_{r-1}\in\R$ are chosen such that $\|\theta^0\|,\ldots,\|\theta^{r_1-1}\|\ll K_0\ll K_1,\ldots, K_{r-1}$, 
then for each $y\in\Ycal\setminus\Zcal$ the vector $\left[ \begin{smallmatrix}1\\\aYu\end{smallmatrix}\right]$ is mapped to a parameter vector $\Theta \left[\begin{smallmatrix}1\\ \aYu\end{smallmatrix}\right]$ with  $p(\cdot|y;\Theta \left[\begin{smallmatrix}1\\ \aYu\end{smallmatrix}\right])$ arbitrarily close to $\delta_{y_m}$. 
\end{proof}

\begin{remark}
In order to prove Lemma~\ref{lemma:star} for any subset $\Zcal\subseteq\Ycal$ it is sufficient to show that (i) the vectors $\{\aYu \colon y\in\Zcal\}$ are affinely independent, and  (ii)  there is a linear map $\Theta$ mapping $\{\left[ \begin{smallmatrix} 1\\\aYu \end{smallmatrix}\right] \colon y\in\Zcal\}$ into the zero vector and  $\left[ \begin{smallmatrix} 1\\\aYu \end{smallmatrix}\right]$ into the relative interior of the normal cone of $Q_\Vcal:=\operatorname{conv}\{\aVu \colon v\in\Vcal\}$ at the vertex $v=y_m$ for all $y\in\Ycal\setminus\Zcal$. 
\end{remark}

\section{Probability Sharing}
\label{section:sharing}

\subsection*{A single directed layer}

Consider an input layer of units $Y_1,\ldots,Y_m$ with bipartite connections directed towards an output layer of units $X_1,\ldots,X_n$. Denote by $\Lcal_{\Xcal,\Ycal}$ the model of conditional distributions defined by this network. 
Recall the definition from eq.~\eqref{eq:layercondrs}. 
Each conditional distribution $p(\cdot|\cdot;\Theta)\in \Lcal_{\Xcal,\Ycal}$ defines a linear stochastic map $F_{\Theta} \colon  q \mapsto \sum_{y\in\Ycal} p(x|y;\Theta) q(y)$ from the simplex $\Delta(\Ycal)$ of  distributions on $\Ycal$ to the simplex $\Delta(\Xcal)$ of distributions on $\Xcal$. 
Here the parameter $\Theta$ corresponds to the parameters $\Theta^l,\theta^l$ of the conditional distributions from eq.~\eqref{eq:layercondrs} for a given $l$.

For any $y\in\Ycal$ and $j\in[m]$, we denote by $y[j]$  the one-dimensional cylinder set $\{y_1\}\times\cdots\times\{y_{j-1}\}\times\Ycal_j\times\{y_{j+1}\}\times\cdots\times\{y_m\}$. Similarly, for any $\Lambda\subseteq[m]$, we denote by $y[\Lambda]$ the cylinder set consisting of all arrays in $\Ycal$ with fixed values $\{y_i\}_{i\in [m]\setminus\Lambda}$ in the entries $[m]\setminus\Lambda$. 

Applying Lemma~\ref{lemma:star} to each output unit of $\Lcal_{\Xcal,\Ycal}$ shows: 

\begin{theorem}
\label{theorem:sharing}
Consider some $\{y^{(s)}\}_{s\in[k]}\subseteq\Ycal$. 
Let $\{j_s\}_{s\in[k]}$ be a multiset and $\{i_s\}_{s\in[k]}$ a set of indices from $[m]$. 
If the cylinder sets $y^{(s)}[j_s]$ are disjoint and $\Zcal$ is a subset of $\Ycal$ containing them, then the image of $\Delta(\Zcal)$ by the family of stochastic maps $\Lcal_{\Ycal,\Ycal}$ contains $\Delta(\Zcal\cup_{s\in[k]} y^{(s)}[\{j_s,i_s\}])$. 
\end{theorem}

This result describes the image of a set of probability distributions by the collection of stochastic maps defined by a DBN layer for all choices of its parameters. 
In turn, it describes part of the DBN representational power contributed by a layer of units. 

\begin{figure}
\setlength{\unitlength}{1cm}
\centering
\includegraphics{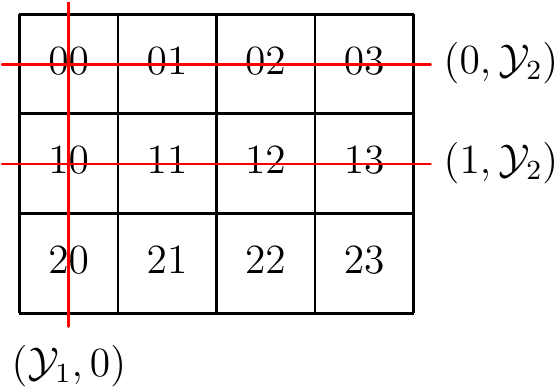}
\caption{Three multi-sharing steps on $\{0,1,2\}\times\{0,1,2,3\}$.}
\label{Figure3}
\end{figure}

\subsection*{A stack of directed layers}

In the case of binary units, sequences of probability sharing steps can be defined conveniently using Gray codes, as done in~\citep[][]{LeRoux2010}. 
A Gray code is an ordered list of vectors, where each two subsequent vectors differ in only one entry. 
A binary Gray code can be viewed as a sequence of one-dimensional cylinder sets. In the non-binary case, this correspondence is no longer given. 
Instead, motivated by Theorem~\ref{theorem:sharing}, we will use one-dimensional cylinder sets in order to define sequences of multi-sharing steps, as shown in Figure~\ref{Figure3}. 

Let $q_i=|\Ycal_i|$ be the cardinality of $\Ycal_i$ for $i\in[n]$, and let $m\leq n$. 
The set $\Zcal=\{0\}\times\cdots\times\{0\}\times\Ycal_{m+1}\times\cdots\times\Ycal_{n} \subseteq\Ycal$ can be written as the disjoint union of $k=\prod_{i=m+2}^{n}|\Ycal_i|$ one-dimensional cylinder sets, as $\Zcal=\dotcup_{s=1}^k y^{(s)}[m+1]$, where $ y^{(s)}=(0,\ldots,\underset{m}{0}\,|\,\underset{m+1}{0},y_{m+2}^{(s)},\ldots,y_{n}^{(s)})$ and $\{(y_{m+2}^{(s)},\ldots,y_{n}^{(s)})\}_{s=1}^k=\Ycal_{m+2}\times\cdots\times\Ycal_{n}$. 

In the following, each  set $y^{(s)}[m+1]$ will be the starting point of a sequence of sharing steps. 
By Theorem~\ref{theorem:sharing}, a directed DBN layer maps the simplex of distributions $\Delta(y^{(1)}[m+1]\cup\cdots\cup y^{(k)}[m+1])$ surjectively to the simplex of distributions $\Delta(y^{(1)}[m+1,1] \cup \cdots\cup y^{(k)}[m+1,k])$. The latter can be mapped by a further DBN layer onto a larger simplex and so forth. 
Starting with $y^{(1)}[m+1]$, consider the sequence 
\begin{equation}
 \renewcommand{\arraystretch}{1.5}
 \renewcommand\arraycolsep{1.5pt}
\begin{array}{l}
(0,0,\ldots,0\,|\,0, y^{(1)}_{m+2},\ldots,{y^{(1)}_{n}})[m+1,1] \\
(0,0,\ldots,0\,|\,0,y^{(1)}_{m+2},\ldots,{y^{(1)}_{n}})[m+1,2] \\
({1},0,\ldots,0\,|\,0,y^{(1)}_{m+2},\ldots,{y^{(1)}_{n}})[m+1,2]
\end{array} 
\end{equation}
continued as shown in Table~\ref{table1}. 
We denote this sequence of cylinder sets by $G^1$, and its $l$-th row (a cylinder set) by $G^1(l)$. 
The union $\cup_{l\in[K]}G^1(l)$ of the first $K$ rows, with $K=1+ q_{1} + q_{1}q_{2}+\cdots+ \prod_{j=1}^{m-1} q_{j}$, is equal to $\Ycal_{1}\times\cdots\times\Ycal_m\times \Ycal_{m+1}\times\{y^{(1)}_{m+2}\}\times\cdots\times\{y^{(1)}_{n}\}$. 

We define $k$ sequences $ G^1,\ldots,G^{k}$ as follows: The first $m$ coordinates of $G^s$ are equal to a permutation of the first $m$ coordinates of $G^1$, defined by shifting each of these $m$ columns cyclically $s$ positions to the right. The last $n-m$ coordinates of $G^s$ are  equal to $(\Ycal_{m+1},y^{(s)}_{m+2},\ldots,y^{(s)}_{n})$. 

We use the abbreviation $\{s+t\}:=(s+t-1)_{\mod(m)}+1$. Within the first $m$ columns, the free coordinate of the $l$-th row of $G^s$ is $s+\kappa$, where $\kappa$ is the least integer with $l\leq \sum_{i=0}^\kappa\prod_{j=s}^{\{s+i-1\}}q_{j}$. Here the empty product is defined as $1$. 
Let $q=\max_{j\in\{1,\ldots,m\}} q_j$. We can modify each sequence $G^s$, by repeating rows if necessary, such that the free coordinate of the $l$-th row of the resulting sequence $\tilde G^s$ is $s+\kappa$, where $\kappa$ is the least integer with $l\leq \sum_{t=0}^\kappa q^t$. This $\kappa$ does not depend on $s$. 

The sequences $\tilde G^s$ for $s\in\{1,\ldots, k\}$ are all different from each other in the last $n-m$ coordinates and have a different `sharing' free-coordinate in each row. 
The union of cylinder sets in all rows of these sequences is equal to $\Ycal_1\times\cdots\times\Ycal_n$. 

\begin{table}
\begin{equation*}
\begin{array}{l}
\renewcommand{\arraystretch}{.8}
\renewcommand\arraycolsep{1.5pt}
{\begin{array}{|c | c| c| c | c| c}
\toprule[.4mm]
\cellcolor[gray]{.9}\Ycal_{1} & 0 & 0 &\;\; 0 \;\;{}&\;\; 0 \;\;{}&\,\phantom{\cdots}\\
\midrule[0.4mm]
\cellcolor[gray]{.7}0 & \cellcolor[gray]{.9}\Ycal_{2}& 0 & 0 & 0 & \\ 
\cellcolor[gray]{.7}1 & \cellcolor[gray]{.9}\Ycal_{2}& 0 & 0 & 0 & \\ 
\cellcolor[gray]{.7}\vdots & \cellcolor[gray]{.9}\vdots & & & & \\ 
\cellcolor[gray]{.7}q_{1}-1 & \cellcolor[gray]{.9}\Ycal_{2}& 0 & 0 & 0 & \\
\midrule[0.4mm]
0 &\cellcolor[gray]{.7} 0 & \cellcolor[gray]{.9}\Ycal_{3} & 0&0& \\
0 &\cellcolor[gray]{.7} 1 & \cellcolor[gray]{.9}\Ycal_{3} & 0&0& \\
\vdots &\cellcolor[gray]{.7}\vdots& \cellcolor[gray]{.9}\vdots & & \\ 
0 &\cellcolor[gray]{.7} q_{2}-1& \cellcolor[gray]{.9}\Ycal_{3} & 0&0& \\
\midrule
1 &\cellcolor[gray]{.7} 0 & \cellcolor[gray]{.9}\Ycal_{3} & 0&0& \\ 
1 &\cellcolor[gray]{.7} 1 & \cellcolor[gray]{.9}\Ycal_{3} & 0&0& \\
\vdots &\cellcolor[gray]{.7}\vdots& \cellcolor[gray]{.9}\vdots & & & \\ 
1 &\cellcolor[gray]{.7} q_{2}-1& \cellcolor[gray]{.9}\Ycal_{3} & 0 & 0 & \\
\midrule
  & & \phantom{q_{3}-1} &\phantom{\Ycal_{4}} &
\end{array}}\\
\phantom{q_{1}-1}\, \vdots \phantom{q_{2}-1 }\vdots\phantom{q_4-1}\\
\renewcommand{\arraystretch}{0.8}
\renewcommand\arraycolsep{1.5pt}
{
\begin{array}{|c|c|c|c|c|c}
  & & & \phantom{\;\; 0 \;\;{}} & \\
\midrule
q_{1}-1 &\cellcolor[gray]{.7} 0 & \cellcolor[gray]{.9}\Ycal_{3} & 0 & 0 & \phantom{\,\cdots}\\ 
q_{1}-1 &\cellcolor[gray]{.7} 1 & \cellcolor[gray]{.9}\Ycal_{3} & 0 & 0 & \\ 
\vdots  &\cellcolor[gray]{.7}\vdots & \cellcolor[gray]{.9}\vdots & & & \\ 
q_{1}-1 &\cellcolor[gray]{.7}q_{2}-1 & \cellcolor[gray]{.9}\Ycal_{3} & 0 & 0 & \\
\midrule[0.4mm]
0 & 0 & \cellcolor[gray]{.7} 0 & \cellcolor[gray]{.9}\Ycal_{4} & 0 &\\ 
0 & 0 & \cellcolor[gray]{.7} 1 & \cellcolor[gray]{.9}\Ycal_{4} & 0 & \\ 
\vdots & \vdots & \cellcolor[gray]{.7}\vdots & \cellcolor[gray]{.9}\vdots & & \\ 
0 & 0 & \cellcolor[gray]{.7}q_{3}-1 & \cellcolor[gray]{.9}\Ycal_{4} &\;\; 0 \;\;{} &\\
\midrule
& & & & &\\
\end{array}} \\
{}\\
\phantom{q_{1}-1}\, \vdots \phantom{q_{2}-1 }\vdots\phantom{q_{4}-1}\\
{}\\
\renewcommand{\arraystretch}{0.8}
\renewcommand\arraycolsep{1.5pt}
{
\begin{array}{|c|c| c| c |c |c|}
  & & \phantom{q_{3}-1} & & & \phantom{\;\; 0 \;\;{}}\\
\midrule
q_{1}-1 & q_{2}-1 & \cdots & q_{m-2}-1 &\cellcolor[gray]{.7}0 & \cellcolor[gray]{.9}\Ycal_{m}\\ 
q_{1}-1 & q_{2}-1 & \cdots & q_{m-2}-1 &\cellcolor[gray]{.7}1 & \cellcolor[gray]{.9}\Ycal_{m}\\ 
\vdots  &  \vdots & & \vdots &\cellcolor[gray]{.7}\vdots & \cellcolor[gray]{.9} \vdots \\ 
q_{1}-1 & q_{2}-1 & \cdots & q_{m-2}-1 &\cellcolor[gray]{.7}q_{m-1}-1 & \cellcolor[gray]{.9}\Ycal_{m}\\
\bottomrule[.4mm]
\end{array}
}
\end{array}
\end{equation*}
\caption{Sequence of one-dimensional cylinder sets.}\label{table1}
\end{table}

\section{Deep Belief Networks}
\label{section:theDBN}

\begin{proposition}
\label{proposition:distrDBN}
Consider a DBN with $L\geq 2$ layers of width $n$, each layer containing units with state spaces of cardinalities $q_1,\ldots,q_n$. 
Let $m$ be any integer with $n\geq m \geq \prod_{j=m+2}^{n}q_j=:k$. 
The corresponding probability model can approximate a distribution $p$ on $\{0,1,\ldots, q_1-1 \}\times\cdots\times\{0,1, \ldots, q_n-1\}$ arbitrarily well whenever the support of $p$ is contained in $\cup_{s\in[k]} \cup_{l\in[L-2]} \tilde G^{s}(l)$. 
\end{proposition}

\begin{proof}
Note that $\prod_{j=m+2}^{n}q_j\leq n \leq 1 + \sum_{j\in[n]}(q_j-1)$. 
By Theorem~\ref{theorem:RBMs} the top RBM can approximate each distribution in the probability simplex on $\{0\}\times\cdots\times\{0\}\times \Xcal_{m+1}\times\cdots\times\Xcal_n$ arbitrarily well. 
By Theorem~\ref{theorem:sharing}, this simplex can be mapped iteratively into larger simplices, according to the sequences $\tilde G^s$ from Section~\ref{section:sharing}. 
\end{proof}

\begin{theorem}
\label{theorem:submDBNs} 
Consider a DBN with $L$ layers of width $n$, each layer containing units with state spaces of cardinalities $q_1,\ldots, q_n$. 
Let $m$ be any integer with $n\geq m \geq \prod_{j=m+2}^{n}q_j$ and  $q= q_1\geq\cdots\geq q_m$. 
If $L\geq 2 + 1 + q+\cdots+q^{S-1} = 2 +\frac{q^{S}-1}{q-1}$, then the DBN model can approximate each distribution in a partition model $\Pcal$ of coarseness $\prod_{j=1}^{m-S}q_{j}$ arbitrarily well. 
\end{theorem}

\begin{proof}
When $L\leq 2$ the result follows from Lemma~\ref{lemma:distrRBMs}. 
Assume therefore that $L\geq 2 + 1 + q+q^2+\cdots+q^r$, $r\geq 0$. 
We use the abbreviation $\{s+t\}:=(s+t-1)_{\mod(m)}+1$. 
Let $k=\prod_{j=m+2}^n q_j$ and $\{  (y_{m+2}^{(s)},\ldots,y_n^{(s)})\colon s\in[k] \}=\Ycal_{m+2}\times\cdots\times\Ycal_n$. 
The top RBM can approximate each distribution from a partition model $\Pcal$ (on a subset of $\Ycal$) arbitrarily well, whose partition blocks are the cylinder sets with fixed coordinate values 
$$y_{ s}=0, y_{ \{s+1\}}=0,\ldots, y_{ \{s+r\}}=0, y_{m+1}, y^{(s)}_{m+2},\ldots,y^{(s)}_{n};$$ 
for all $y_{m+1}\in\Ycal_{m+1}$, for all $s\in[k]$. 
After $L-2$ probability sharing steps starting from $\Pcal$, the DBN can approximate the distributions from the partition model arbitrarily well, whose partition blocks are the cylinder sets with fixed coordinate values 
$$y_{ s}, y_{ \{s+1\}},\ldots, y_{\{s+r\}}, y_{m+1}, y^{(s)}_{m+2}, \ldots, y^{(s)}_{n};$$
for all possible choices of $y_{ s}, y_{ \{s+1\}}, \ldots, y_{ \{s+r\}}, y_{m+1}$, for all $s\in[k]$. 
The maximal cardinality of such a block is $q_{1}\cdots q_{m-r-1}$, and the union of all blocks equals $\Ycal$. 
\end{proof}

\begin{proof}[Proof of Theorem~\ref{theorem:main}]
The claim follows bounding the divergence of the partition models described in Theorem~\ref{theorem:submDBNs}. 
\end{proof}

As a corollary we obtain the following bound for the expectation value of the divergence from distributions drawn from a Dirichlet prior, to the DBN model. 
\begin{corollary}
\label{corollary:expectations}
The expectation value of the divergence from a probability distribution $p$ drawn from the symmetric Dirichlet distribution $\Dir_{(a,\ldots,a)}$ to the model $\DBN$ from Theorem~\ref{theorem:main} is bounded by 
\begin{equation*}
\int_{\Delta} D(p\| \DBN) \,\Dir_{(a,\ldots,a)}(p)\, \rd p \leq \left(\psi(a+1) - \psi( c a + 1) + \ln(c) \right) \log(e),  
\end{equation*}
where $c= \prod_{j\in[m-S]} q_j$, $\psi$ is the digamma function, and $e$ is Euler's constant.  
\end{corollary}

\begin{proof}
This is a consequence of analytical work~\citep{wupes2012} on the expectation value of Kullback-Leibler divergences of standard probability models, 
applied to the partition models described in Theorem~\ref{theorem:main}. 
\end{proof}

\appendix

\section{Small Experiments}
\label{section:experiments}
We run some computer experiments, 
not with the purpose of validating the quality of our bounds in general, but with the purpose of giving a first empirical insight. 
It is important to emphasize that numerical experiments evaluating the divergence from probability models defined by neural networks are only feasible for small networks, since otherwise the model distributions are too hard to compute~\cite[see, e.g.,][]{Long2010}. 
For large models one still could try to sample the distributions and replace the divergence by a proxy, like the discrepancy of low-level statistics, but here we will focus on small networks. 

We generate artificial data in the following natural way: 
For a given visible state space $\Xcal$ and the corresponding probability simplex $\Delta(\Xcal)$, 
we draw a set of distributions $\{ p^i\in\Delta(\Xcal) \colon i=1,\ldots, T \}$ from the Dirichlet distribution $\Dir_{(a,\ldots,a)}$ on $\Delta(\Xcal)$. 
For the purpose of our experiments, we choose the concentration parameter $a$ in such a way that the Dirichlet density is higher for low-entropy distributions (most distributions in practice have relatively few preferred states and hence a small entropy). 
Next, for each $i=1,\ldots,T$, we generate $N$ i.i.d.~samples from $p^i$, which results in a data vector $X^i = (X^i_1,\ldots, X^i_N) \in\Xcal^N$ with empirical distribution $P^i=\frac{1}{N}\sum_{j=1}^{N} \delta_{X^i_j}$. 

A network $\Ncal$ (with visible states $\Xcal$) is then tested on all data sets $X^i$, $i = 1,\ldots, T$. 
For each data set we train $\Ncal$ using contrastive divergence (CD)~\citep{Hinton2002,Hinton2006} and maximum likelihood (ML) gradient. 
This gives us a maximum likelihood estimate $p_{ \theta_i}$ of $P^i$ within $\Ncal$. 
Finally, we compute the Kullback-Leibler divergence $D(P^i\| p_{\theta_i})$, the maximum value over all data sets $\text{\textsf{max}}_{\text{\textsf{CD+ML}}} = \max_{i=1,\ldots,T} D(P^i\| p_{\theta_i})$, and the mean value over all data sets $\text{\textsf{mean}}_{\text{\textsf{CD+ML}}}=\frac{1}{T}\sum_{i=1}^T D(P^i\| p_{\theta_i})$. 
We do not need cross validation, or $D(p^i\|p_{\theta_i})$, because we are interested in the representational power of $\Ncal$, rather than on its generalization properties. 

\begin{figure}
\includegraphics[width=7cm]{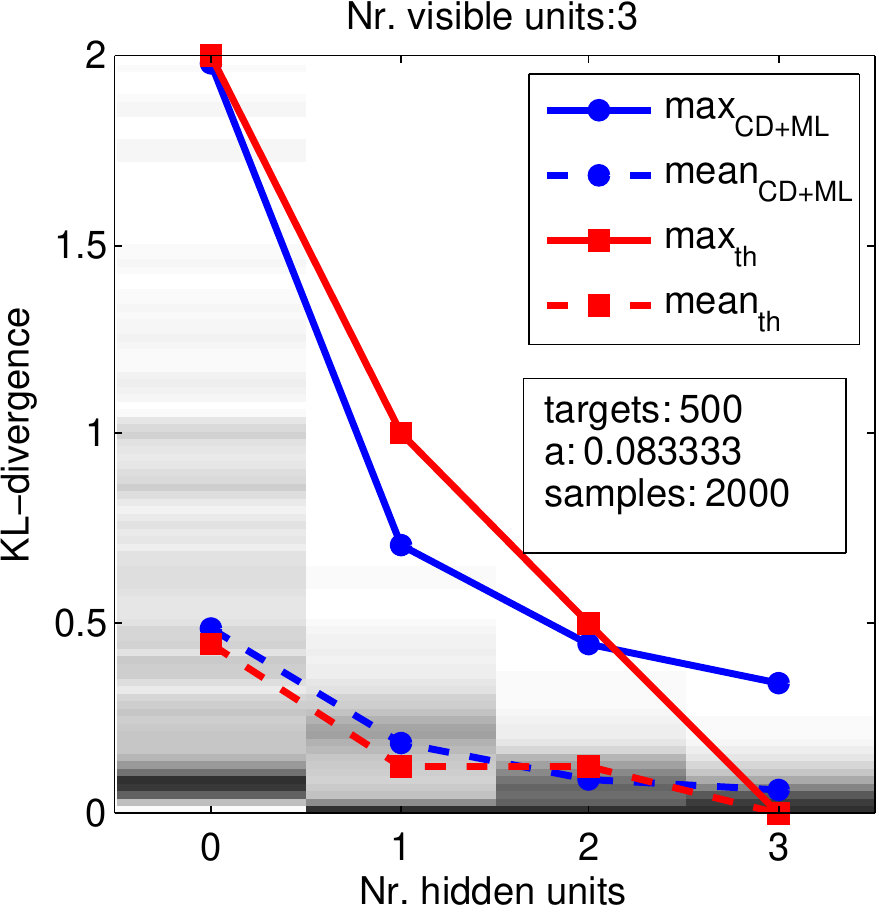}\quad
\includegraphics[width=7cm]{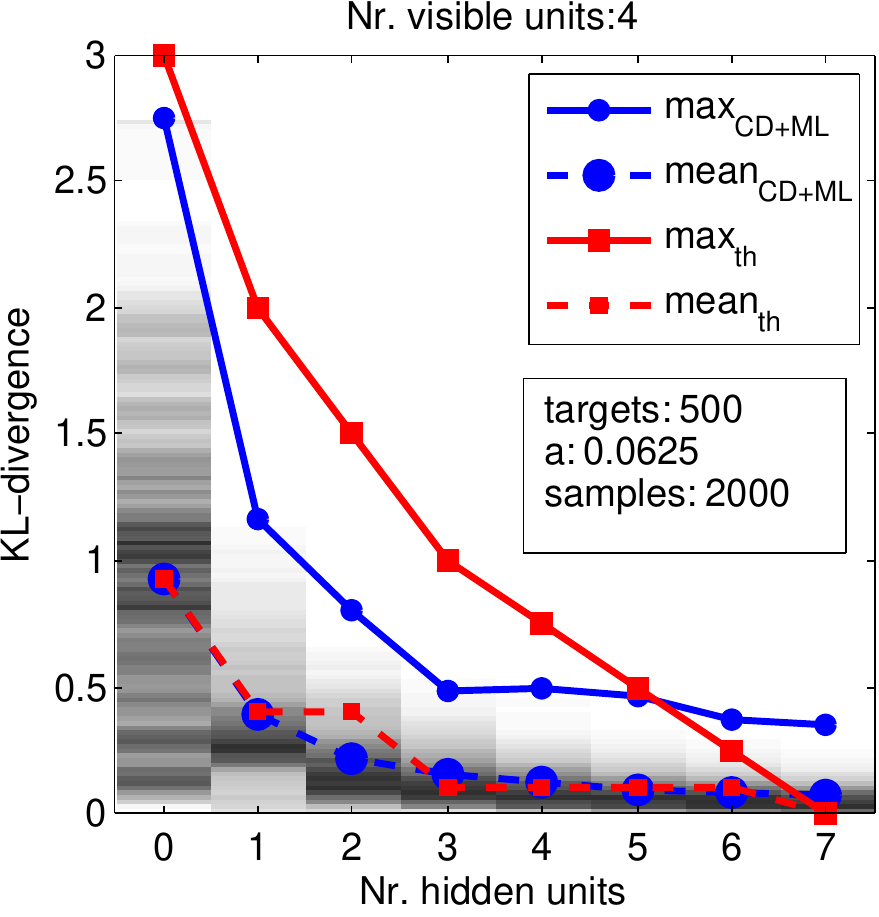}
\caption{Empirical evaluation of the representational power of small binary RBMs. 
The gray shading indicates the frequency at which the target distribution $P^i$ had a given divergence from the trained RBM distribution $p_{\theta_i}$ (the darker a value, the more frequent). 
The lines with round markers show the mean divergence (dashed) and  maximal divergence (solid) over all target distributions. 
The lines with square markers show the theoretical upper bounds of the mean divergence (dashed) and maximal divergence (solid) over the continuum of all possible target distributions drawn from the symmetric Dirichlet distribution $\Dir_{(a,\ldots,a)}$.  
}
\label{Figure4}
\end{figure}

We note that the number of distributions which have the largest divergence from $\Ncal$ is relatively small, and hence the random variable $\text{\textsf{max}}_{\text{\textsf{CD+ML}}}$ has a large variance (unless the number of data sets tends to infinity, $T\to \infty$). 
Moreover, we note that it is hard to find the best approximations of a given target $P^i$. 
Since the likelihood function $L_{X^i}(\theta)=\prod_{j=1}^N p_{\theta}(X^i_j)$ has many local maxima, the distribution $p_{\theta_i}$ is often not a global maximizer of $L_{X^i}$, 
even if training is arranged with many parameter initializations. 
Many times the estimated value $p_{\theta_i}$ is a good local minimizer of the divergence, but sometimes it is relatively poor (especially for the larger networks). 
This contributes again to the variance of $\text{\textsf{max}}_{\text{\textsf{CD+ML}}}$. 
The mean values $\text{\textsf{mean}}_{\text{\textsf{CD+ML}}}$, on the other hand, are more stable. 

\begin{figure}
\centering
\includegraphics[width=7cm]{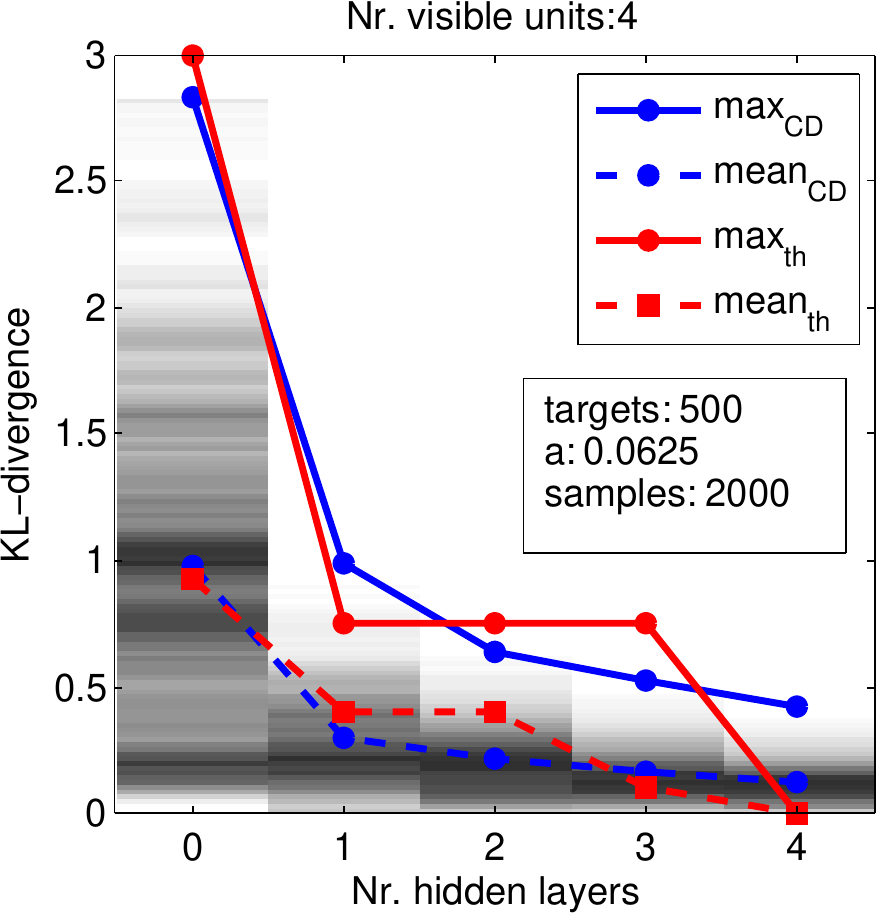}
\caption{Empirical evaluation of the representational power of small binary DBNs. 
The details are as in Figure~\ref{Figure4}, whereby the theoretical upper bounds shown here for the maximal and mean divergence, $\text{\textsf{max}}_{\text{\textsf{th}}}$ and $\text{\textsf{mean}}_{\text{\textsf{th}}}$, are a combination of our results for RBMs and DBNs. 
}
\label{Figure5}
\end{figure}

Figure~\ref{Figure4} shows the results for small binary RBMs with $3$ and $4$ visible units, 
and Figure~\ref{Figure5} shows the results for small constant-width binary DBNs with $4$ visible units. 
In both figures the maximum and mean divergence is captured relatively well by our theoretical bounds. 
The empirical maximum values have a well recognizable discrepancy from the theoretical bound. 
This is explained by the large variance of $\text{\textsf{max}}_{\text{\textsf{CD+ML}}}$, given the limited number of target distributions used in these experiments. 
Finding a maximizer of the divergence (a data vector $X\in\Xcal^N$ that is hardest to represent) is hard. 
Most target distributions can be approximated much better than the hardest distributions. 
A second observation is that with increasing network complexity (more hidden units), finding the best approximations of the target distributions becomes harder (even increasing the training efforts). 
This causes the empirical maximum divergence to actually surpass the theoretical bounds. 
In other words, although the models are in principle able to approximate the targets accurately, according to our theoretical bounds, in practice they may not, because of the difficult training, and their capacity remains wasted. 
The empirical mean values, on the other hand, have a much lower variance and are captured quite accurately by our theoretical bounds. 

\subsection*{Acknowledgment}
The author was supported in part by DARPA grant \mbox{FA8650-11-1-7145}. 
The author completed the revision of the original manuscript at the Max Planck Institute for Mathematics in the Sciences, Inselstra\ss e 22, 04103 Leipzig, Germany. 


\end{document}